\documentclass[accepted]{uai2022} 

\usepackage[american]{babel}
\usepackage{graphicx}
\graphicspath{ {./img/} }
\usepackage{amssymb}
\usepackage{amsmath}
\usepackage{comment}
\usepackage{natbib}
\usepackage{hyperref}

\newcommand{\norm}[1]{\left\lVert#1\right\rVert}

\DeclareMathOperator{\Tr}{Tr}

\usepackage{soul}
\usepackage[dvipsnames]{xcolor}

\newcommand{\bfit}[1]{\boldsymbol{#1}}
\usepackage{thmtools}
\newtheorem{prop}{Proposition}
\newtheorem{definition}{Definition}
\newtheorem{lemma}{Lemma}
\newtheorem{proof}{Proof}
\newtheorem{proof2}{Proof}
\newtheorem{proofsketch}{Proof Sketch}

\usepackage{natbib} 
\usepackage{mathtools} 
\usepackage{booktabs} 
\usepackage{tikz} 

\title{Probabilistic Embeddings with Laplacian Graph Priors}

\author[1]{\href{mailto:<jj@example.edu>?Subject=Your preprint}{Väinö Yrjänäinen}{}}
\author[1]{Måns Magnusson}
\affil[1]{%
    Department of Statistics\\
    Uppsala University\\
    Uppsala, Sweden
}
\begin{document}
\maketitle

\begin{abstract}
We introduce probabilistic embeddings using Laplacian priors (PELP). The proposed model enables incorporating graph side-information into static word embeddings. We theoretically show that the model unifies several previously proposed embedding methods under one umbrella. PELP generalises graph-enhanced, group, dynamic, and cross-lingual static word embeddings. PELP also enables any combination of these previous models in a  straightforward fashion. Furthermore, we empirically show that our model matches the performance of previous models as special cases. In addition, we demonstrate its flexibility by applying it to the comparison of political sociolects over time. Finally, we provide code as a TensorFlow implementation enabling flexible estimation in different settings.
\end{abstract}

\section{Introduction}

Word embeddings are a common approach to quantify semantic properties of language, such as the distributional similarity and relatedness of words \citep{allen2019vec}. Word embeddings have been utilized extensively for transfer learning in predictive models \citep{kim2014convolutional, iyyer2014political}, other downstream tasks \citep{lilleberg2015support, ma2015using}, and more recently, for analyzing and measuring meaning and bias in large corpora in the social sciences and the humanities \citep{tahmasebi2018survey,chaudhuri2019understanding,nguyen2020we,Rodriguez2022WordEW}. As an example, \citet{stoltz2020cultural} use the "meaning-space" spanned by word embeddings to navigate different meaning structures in the study of societal discourse. Hence, word embeddings have become a methodology to explore shared understandings and cultural meaning \citep{garg2018word,kozlowski2019geometry,bodell2019interpretable,rodman_2020}.

Probabilistic word embedding models have emerged as a way to model textual data in scientific applications \citep{rudolph2016exponential, bamler2017dynamic}. The advantages include flexible inclusion of prior knowledge, explicit handling of uncertainty, straightforward estimation, and usefulness in decision-making \citep{ghahramani2015probabilistic}. More versatile and flexible word embedding methods allow for increasingly complex research use cases, such as modelling temporal structure and different authors in the data \citep{stoltz2020cultural}. Indeed, the static Bernoulli word embeddings of \citet{rudolph2016exponential} have been extended to new probabilistic models handling specific structures of interests to social scientists, such as dynamic embeddings \citep{rudolph2018dynamic, bamler2017dynamic}, grouped-based embeddings \citep{rudolph2017structured}, and interpretable embeddings \citep{bodell2019interpretable}. Recently, also contextualized embeddings using transformers have started using probabilistic frameworks to model textual data \citep[e.g. see ][]{hofmann2020dynamic}. Including additional information, such as word graphs, authors, or cross-lingual structures, is of general interest as it further facilitates modelling different phenomena and relations. However, there is currently lacking more general frameworks for this type of modelling.

\subsection{Probabilistic word embeddings} 
\label{background}

As a part of a family of probabilistic embedding methods, \citet{rudolph2016exponential} formulated the Bernoulli Embeddings model. In the model, a corpus $(v_1,\dots,v_N)$ of $N$ words $v_i \in W$ is represented as one-hot vectors $x = (x_1,\dots,x_N)$. Each word at index $i$ has a symmetric $M$-sized context window $\mathrm{x}_{i} = \{x_{i-M}, \dots , x_{i+M} \} \setminus x_i$. The conditional distribution of $x_{iv}$ is 
\begin{equation} \label{CBOWbernoulliConditional}
x_{iv} |  \mathrm{x}_{i} \sim \textrm{Bernoulli}(p_{iv})\textrm{ }
\end{equation}
where
\begin{equation} \label{CBOWbernoulliConditional}
p_{iv} = \sigma(\eta_{iv})
\end{equation}
where $\sigma(\cdot)$ is the logistic function. Similarly to \citet{mikolov2013distributed}, the Bernoulli distribution is used to approximate the Categorical distribution in the model.

Each word type $v$ has its corresponding word and context vectors $\rho_v, \alpha_v \in \mathbb{R}^D$, $D \in \mathbb{N}$. The value $\eta_{iv}$ is the inner product between the embedding vector $\rho_v$ and the sum of the context vectors of the words surrounding position $i$
\begin{equation}
\label{eq:cbow_eta}
\eta_{iv} = \rho^T_v \sum_{w \in \mathrm{x}_{i} } \alpha_w\,.
\end{equation}
A priori, both the word and the context vectors follow a spherical multivariate Gaussian distribution as
\begin{equation}
\label{eq:be_rhoalpha_prior}
\rho_w, \alpha_w \sim \mathcal{N}(0, \gamma^{-1} \bfit{I})\,,
\end{equation}
for all $w \in W$, where $\bfit{I} \in \mathbb{R}^{D }$ and $\gamma \in \mathbb{R}_+$. The parameters $\rho, \alpha$ constitute $\theta = [\rho, \alpha] \in \mathbb{R}^{2V \times D}$.

With large corpora, computing expectations over the full posterior is expensive. Instead, MAP estimation is employed, and estimates for other quantities are calculated based on point estimates \citep{rudolph2018dynamic}.
The likelihood consists of so-called positive ($x_{ps}$) and negative ($x_{ns}$) samples
\begin{equation}
\label{eq:cbow_loss_1}
\begin{split}
\log p(x \mid \theta) &= \log p(x_{ps} \mid \theta) + \log p(x_{ns} \mid \theta) \\
&= \sum_{i=1}^N \log(p_{iv_i}) + \sum_{i=1}^N \sum_{w \sim \mathcal{W}}^K \log(1 - p_{iw})
\end{split}
\end{equation}
where $v_i$ is the word type at index $i$, $K$ is the number of negative samples, and $\mathcal{W}$ is the smoothed empirical distribution of the word types. For each position $i$ in the data, $K$ word indices $w$ are sampled from $\mathcal{W}$. The CBOW likelihood is part of the Bernoulli Embeddings posterior
\begin{equation}
\begin{split}
\log p_{BE}(\theta \mid x) &= \log p(x \mid \theta) + \log p(\theta) + C  \\
\end{split}
\end{equation}
where the prior $p(\theta)$ is defined as in Eq. (\ref{eq:be_rhoalpha_prior}). Independently, \citet{bamler2017dynamic} presented a similar probabilistic embedding model using the SGNS likelihood.

\subsection{Laplacian Graph Priors}

The Laplacian matrix $\bfit{L}$ of a graph $\mathcal{G}$, with $E$ being the set of edges, is defined as $\bfit{L} = \bfit{D} - \bfit{A}$, where $\bfit{A}$ is the adjacency matrix of the graph, and $\bfit{D}$ is the degree matrix of the graph.
The Laplacian can be extended for weighted and signed graphs and, by definition, all of these Laplacian matrices are positive semi-definite \citep{merris1994laplacian, das2005sharp}.
By summing the Laplacian and any positive diagonal matrix, most commonly a scaled identity matrix, the augmented Laplacian becomes strictly positive-definite, and is a valid precision matrix $\bfit{\Sigma}^{-1}$ for the multivariate Gaussian distribution \citep{strahl2019scalable}. We denote such Laplacian precision matrices with $\bfit{L}_+$
\begin{equation}  \label{Lplus}
  \bfit{L}_+ = \lambda_1 \bfit{L} + \lambda_0 \bfit{{D}}
\end{equation}
where $\lambda_1, \lambda_0 \in \mathbb{R}_+$. Any side-information that is presentable as an undirected graph can then be used as a prior. Laplacian priors have been used in this way for tasks such as matrix factorization \citep{cai2010graph}, image classification \citep{zheng2010graph} and image denoising \citep{liu2014progressive}.

\subsection{Graph Side-Information in Word Embeddings}
There have been several attempts to incorporate side-information into word embeddings. \citet{faruqui2014retrofitting} improved the word embedding performance on several tasks by post hoc corrections, or \textit{retrofitting}, the embeddings using on the WordNet graph. Later, \citet{tissier2017dict2vec} improved word embedding accuracy by augmenting it with side information from dictionaries with the \texttt{dict2vec} embedding model.
\begin{table}[h!]
\centering
\begin{tabular}{c c } 
\hline
\textbf{Word} & \textbf{Definition} \\
\hline
conscious & \textbf{aware} of one's own existence\\ aware &\textbf{conscious} or having knowledge \\
\hline
\end{tabular}
\caption{In the dict2vec method, the words \textit{conscious} and \textit{aware} are both mentioned in each others' definitions, and thus they form an edge in the graph.}
\label{table_word_definition}
\end{table}
Their method was based on reciprocal occurrence in the definitions of words. They then amend their likelihood with the linked pairs $\mathcal{G}$ as positive samples $x_{ps,\mathcal{G}}$, yielding the following likelihood function
\begin{equation}
\label{eq:dict2vec}
p(x \mid \theta) = p(x_{ps} \mid \theta) + p(x_{ns} \mid \theta) + p(x_{ps,\mathcal{G}} \mid \theta)
\end{equation}
This approach substantially boosted the performance of the model on word similarity tasks \citep{tissier2017dict2vec}.

\subsection{Dynamic and Grouped Embedding Models} \label{DynamicGrouped}
Based on the Bernoulli Embeddings model, \citet{rudolph2018dynamic} presented the Dynamic Bernoulli model (DBM), where each word vector is modeled as a random walk. A priori, the word vectors at $t=0$ are distributed
\begin{equation}
\label{eq:dynamic_bernoulli_1}
    \rho_{0,w} \sim \mathcal{N}(0, \gamma_0^{-1}\bfit{I})\,,
\end{equation}
for all word types $w \in W$ and for all timesteps $t \in \{1,2 \dots , T\}$
\begin{equation}
\label{eq:dynamic_bernoulli_2}
    \rho_{t,w} \sim \mathcal{N}(\rho_{t-1}, \gamma_1^{-1} \bfit{I})\,,
\end{equation}
for all word types $w \in W$, while $\gamma_0, \gamma_1 \in \mathbb{R}_+$ and $\bfit{I}$ is a $D$-dimensional identity matrix. The context vectors are shared between the timesteps, and are distributed as in Eq. (\ref{eq:be_rhoalpha_prior}) for all word types $w \in W$. The likelihood is similar to Eq (\ref{eq:cbow_loss_1}), using the timestep of each data point in question.

The hierarchical grouped Bernoulli model (GBM) \citep{rudolph2017structured} is similar to DBM, but instead of timesteps, the corpus is divided into two or more groups $s \in S$. Each word $w$ corresponds to a group word vector $\rho_{0,w}$, and the actual word vectors $\rho_{s,w}, s \in S$ are normally distributed with the group vector as the mean
\begin{equation}
\label{eq:grouped_bernoulli_1}
    \rho_{0,w} \sim \mathcal{N}(0, \gamma_0^{-1} \bfit{I})
\end{equation}
\begin{equation}
\label{eq:grouped_bernoulli_2}
    \rho_{s,w} \sim \mathcal{N}(\rho_{0, w}, \gamma_1^{-1} \bfit{I})
\end{equation}
where $\gamma_0, \gamma_1 \in \mathbb{R}_+$ and $\bfit{I}$ is a $D$-dimensional identity matrix. The context vectors are shared between the groups, and are distributed as in Eq. (\ref{eq:be_rhoalpha_prior})
for all word types $w \in W$. The likelihood is similar to in Eq (\ref{eq:cbow_loss_1}), using the respective group of each data point in question.

\subsection{Cross-Lingual Word Embeddings} \label{cross-lingualTheory}
Cross-lingual word embeddings embed the words of two or more languages $\{A, B, \dots \}$ into a shared embedding space $\mathbb{R}^{D}$. This can be done using multilingual corpora and a graph $\mathcal{G}_{AB \dots } = (W_{AB \dots } , E_{AB \dots })$ of translation pairs in the combined set of words $W_{AB \dots }$ in languages $A,B, \dots$. Commonly used \textit{mapping methods} first separately train the monolingual embeddings, for example using \texttt{word2vec}, and then find the mapping $\mathbb{R}^D \to \mathbb{R}^D$ that minimizes the squared difference between the word translation pairs $E_{AB} = \{(w_{A,1}, w_{B,1}),  \dots, (w_{A,N}, w_{B,N}) \}$, where $w_{X,i}$ is a word type in language $X$ \citep{ruder2019survey}. 

Linear transformations have been found to work well in practice, e.g. to outperform feedforward neural networks
\citep{mikolov2013exploiting,ruder2019survey}.
The optimal linear transformation minimizes
\begin{equation}
\label{eq:mapping_methods}
\bfit{\hat{W}} = \arg \min_{\bfit{W}} \sum_{(v,w) \in E_{AB}} \norm{ \bfit{W} \rho_{A,v} - \rho_{B,w}}^2\,,
\end{equation}
where $\rho_{A,v}$ is the word vector for $v$ in the language A and $\rho_{B,w}$ is the word vector for $w$ in the language B. Some variations further restrict the linear transformation matrix to be orthogonal, which has been found to improve the Bilingual Lexicon Induction performance of the embeddings \citep{xing2015normalized}, while preserving monolingual distances \citep{smith2017offline}.
The optima of the linear and orthogonal mappings have closed-form solutions \citep{artetxe2016learning}. Another common cross-lingual method is the pseudo-multilingual corpora method. It uses a regular embedding method and randomly swaps translation pairs $E_{AB}$ with $\frac{1}{2}$ probability \citep{gouws2015simple}. This approach has been proven to be equivalent to using parameter sharing for the word vectors of the translation pairs \citep{ruder2019survey}.


\subsection{Contributions and Limitations}
This paper introduces the probabilistic embeddings with the Laplacian priors (PELP) model. We summarize the main contributions as:
\begin{enumerate}
    \item We introduce the PELP model that can incorporate any undirected graph-structured side-information using either the continuous bag-of-words (CBOW) or the skip-gram with negative samples (SGNS) likelihood. 
    \item We introduce a cross-lingual PELP model to handle multiple languages.
    \item We prove that PELP generalizes many previous word embedding models, both monolingual and cross-lingual. 
    \item We show that PELP enables the combination of multiple previous embedding models into joint models (e.g. dynamic and group embedding models).
    \item We provide a TensorFlow implementation to estimate probabilistic word embeddings with different Laplacian priors. The implementation enables GPU parallelism on large corpora for any weighted and unweighted Laplacian.
\end{enumerate}

Although the proposed PELP model is flexible and general, it also has limitations. The main limitation is that the edges in graph Laplacian priors are necessarily symmetric if we interpret the model as a prior. Hence, directed graph side-information must be treated as undirected if included in the model formally as a prior. Furthermore, the embeddings presented in this paper are static, compared to contextualized embeddings from transformer neural networks such as in \citet{hofmann2020dynamic}.

\section{Probabilistic Word Embeddings with Laplacian Priors}
\label{PELP_section}
We propose Probabilistic Embeddings with Laplacian priors (PELP). The model combines the Bernoulli embeddings model by \citet{rudolph2016exponential} as specified in Eq. (\ref{CBOWbernoulliConditional}) with a graph Laplacian prior, both for the CBOW model and the SGNS model. The model is then expanded to a cross-lingual setting. This chapter describes the models and shows their properties and theoretical generalization of previous models.

\subsection{Monolingual PELP}
\label{definition:pelp_definition}
%
Let $\mathcal{G}$ be a graph $(W, E)$ with a set of edges ${E \subseteq \{\{v,w\}\mid v,w \in W\;{\textrm {and}}\;v\neq w\}}$ in the set of word types $W$. We can formulate a Laplacian matrix according to Eq. (\ref{Lplus}) to arrive at a positive-definite Laplacian. As each word type $v$ has its dedicated word and context vectors $\rho_v, \alpha_v$, we can set this as the precision matrix for the prior distribution of $\theta = [\rho, \alpha]$ as
\begin{equation}
\label{eq:pelp_dimension_wise}
p(\theta) = \prod_{i=1}^{D} \mathcal{N}(\theta_{:i} \mid 0, \bfit{L}_+^{-1} )\,,
\end{equation}
where $D$ is the dimensionality of the embeddings, and $\bfit{L}_+$ is the augmented Laplacian matrix as defined in Eq. (\ref{Lplus}). 

In its logarithmic form, the Laplacian prior becomes a sum over the squared differences between the parameters $\theta_v, \theta_w$
\begin{equation}
\label{eq:pelp_sum_of_squares}
\begin{split}
\log p(\theta)  = & -\frac{1}{2} \Tr(\theta^T \bfit{L}_+ \theta) + C \\
= &-\frac{\lambda_1}{2} \sum_{(v,w) \in E} \norm{\theta_v - \theta_w}^2 \\
&-\frac{\lambda_0}{2} \sum_{v \in W} \left( \norm{\rho_v}^2 + \norm{\alpha_v}^2 \right) + C\\
\end{split}
\end{equation}
For weighted graphs $\mathcal G = (W,E)$, ${E \subseteq \{\{v,w\} \to a_{ij} \mid v,w \in W\;{\textrm {and}}\;v\neq w\}}$, this prior is
\begin{equation}
\label{eq:pelp_sum_of_squares}
\begin{split}
\log p(\theta)  = & -\frac{1}{2} \Tr(\theta^T \bfit{L}_+ \theta) + C \\
= &-\frac{1}{2} \sum_{(v,w) \in E} a_{ij} \norm{\theta_v - \theta_w}^2 \\
&-\frac{\lambda_0}{2} \sum_{v \in W} \left( \norm{\rho_v}^2 + \norm{\alpha_v}^2 \right) + C\\
\end{split}
\end{equation}
The full posterior is
\begin{equation}
\label{eq:pelp_loss_terms_sum}
\log p(\theta \mid x) = \log p(x \mid \theta) + \log p(\theta)
\end{equation}
where the likelihood, $\log p(x \mid \theta)$, can either be the CBOW or the SGNS likelihood. In either case, negative sampling (Eq. \ref{eq:cbow_loss_1}) is used.

\subsection{Cross-lingual PELP}
\label{definition:cross-lingual_pelp}

For cross-lingual data, PELP can easily be extended to the cross-lingual setting. For two langauges $A$ and $B$, cross-lingual PELP has the parameters $\rho = (\rho_A, \rho_B)$ and $\alpha = (\alpha_A, \alpha_B)$, and $\theta = (\rho_A, \rho_B, \alpha_A,\alpha_B)$, which can also be organized into $\theta_A = (\rho_A, \alpha_A)$ and $\theta_B = (\rho_B, \alpha_B)$. The likelihood consists of the monolingual likelihoods
\begin{equation}
\label{eq:cross-lingual_pelp}
\begin{split}
\log p(x \mid \theta) = & \log p(x_A \mid \theta_A) \\
+& \log p(x_B \mid \theta_B)  \\
\end{split}
\end{equation}
where $x_A, x_B$ are the data for languages $A$ and $B$, respectively.

The graph $\mathcal G_{AB} = (W_{AB}, E_{AB})$ for the Laplacian consists of translation pairs, but can also incorporate other graph-information. The cross-lingual Laplacian prior $p(\theta_A, \theta_B)$ can be formulated as
\begin{equation}
\label{eq:translation_pairs}
\begin{split}
&\log p(\theta_A, \theta_B) = \\
&-\frac{\lambda_1}{2} \sum_{(v,w) \in E_{AB}} \norm{ \rho_{A_v} - \rho_{B_w}}^2 \\
&-\frac{\lambda_1}{2} \sum_{(v,w) \in E_{AB}} \norm{ \alpha_{A_v} - \alpha_{B_w}}^2 \\
&-\frac{\lambda_0}{2} \sum_{v \in W} \norm{ \rho_{v}}^2 + \norm{ \alpha_{v}}^2 + C
\end{split}
\end{equation}
where $\lambda_1 \in \mathbb{R}_+$ is a hyperparameter, and $C$ is a normalizing constant. The full posterior is
\begin{equation}
\log p(\theta \mid x) = \log p(x \mid \theta) + \log p(\theta_A, \theta_B)
\end{equation}

\subsection{Generalization of Previous Methods}
\label{sec:relationship_with_previous_word_embeddings}
Many word embedding models can be seen as special cases of specific Laplacian priors or are closely related to a Laplacian prior with a particular graph. In the following section, we will use the following assumptions to prove that many previous word embedding methods are special cases of the PELP model.
\begin{itemize}
    \item[(a)] The likelihood functions of two models that are being compared are identical, i.e. both are CBOW or SGNS as defined in Section \ref{background}.
    \item[(b)] There exists a \textit{maximum a posteriori} estimate $\hat{\theta} = {\arg \max}_\theta \log (\theta|x)$.
\end{itemize}
By penalizing the squared difference between word pairs, PELP bears significant similarities with Faruqui's (\citeyear{faruqui2014retrofitting}) retrofitting. Unlike \citet{faruqui2014retrofitting}, though, the graph-structure additional constraint is applied during estimation, which unlike Faruqui's method can indirectly affect words outside the side-information graph. In addition, PELP is closely related to \texttt{dict2vec} \citep{tissier2017dict2vec}. While \texttt{dict2vec} adds positive samples of the strong pairs ad-hoc, the graph structure of these pairs can be used as a Laplacian prior. We formulate this theoretical relationship with the following proposition.

\begin{restatable}{prop}{propositionone}
\label{prop:dict2vec}
Let the \texttt{dict2vec} model be defined as in Eq. (\ref{eq:dict2vec}), augmented with  $\rho_v, \alpha_v \sim N(0, \gamma_0^{-1} \bfit{I})$ priors for all word types $v \in W$. Also, let the PELP model be defined as in Definition \ref{definition:pelp_definition}. Further assume that assumptions (a)-(b) hold, and let $\mathcal{G}$ be a graph shared by the \texttt{dict2vec} and the PELP model with the augmented Laplacian $\bfit{L}_+ = \bfit{L} + \bfit{D}$. Then, for any $\hat{\theta}$ in dict2vec model, there exists a PELP model with a specific weighted Laplacian $\bfit{L}^\star$ and the augmenting diagonal matrix $\bfit{D}^\star$ with the same stationary points as the \texttt{dict2vec} model.
\end{restatable}

\begin{proof2}
See appendix.
\end{proof2}

Some priors previously proposed in the literature, can also be seen as special cases of the PELP model. A graph representation of the DBM prior \citep{rudolph2018dynamic} would be a chain of each word across timesteps. This is essentially a graph prior with a specific Laplacian $\bfit L^\star$. We summarize this result with the following proposition.
\begin{restatable}{prop}{propositiontwo} \label{prop:dynamic}
Let the Dynamic Bernoulli model (DBM) be defined as in Eq. (\ref{eq:dynamic_bernoulli_1}–\ref{eq:dynamic_bernoulli_2}). Also, let the PELP model be defined as in Eq. (\ref{eq:pelp_dimension_wise})-(\ref{eq:pelp_loss_terms_sum}), so that the precision matrix is a Laplacian of a graph plus a diagonal matrix. Assuming the parameters $\lambda_0, \lambda_1$ are the same in both models and that assumption (a)-(b) holds, then the DBM and the PELP model are identical.
\end{restatable}
\begin{proof2}
See appendix.
\end{proof2}

Similarly to the DBM, the GBM model of section \ref{DynamicGrouped} can also bee seen as a special case of the PELP model, using another specific Laplacian prior $\bfit L^\star$. We summarize this theoretic results in the following proposition.
\begin{restatable}{prop}{propositionthree} \label{prop:grouped}
Let the grouped Bernoulli model (GBM) as defined in Eq. (\ref{eq:grouped_bernoulli_1})-(\ref{eq:grouped_bernoulli_2}) for groups $s \in S$, let PELP model be defined as in (\ref{eq:pelp_dimension_wise})-(\ref{eq:pelp_loss_terms_sum}), and let $\mathcal{G}$ be a graph for the Laplacian prior of PELP. Assuming (a)-(b), and that
\begin{itemize}
    \item[(c)] the graph $\mathcal{G}$ consists only of fully connected subgraphs of each group $s \in S$,
\end{itemize}
then there exists PELP model with a precision matrix $\bfit{L}^\star + \bfit{D}^\star$ to which the GBM is identical.
\end{restatable}
\begin{proof2}
See appendix.
\end{proof2}

In the cross-lingual case, the full data set consists of data sets in two or more languages $\{A, B \dots \}$. According to \citet{ruder2019survey}, some of the cross-lingual approaches are similar given the same corpus and optimization approach are used. We extend their result and formulate the following two propositions, showing that \citet{gouws2015simple} and \citet{artetxe2016learning} are two different special cases of the cross-lingual PELP model. We summarize these result in the following two propositions.
\begin{restatable}{prop}{propositionfour} \label{prop:cross-lingual_a}
Let the cross-lingual PELP be defined as in Definition \ref{definition:cross-lingual_pelp}, let the pseudo-multilingual corpora model (PML) of \citet{gouws2015simple} be defined as in Section \ref{cross-lingualTheory}, and let $\mathcal{G_{AB}}$ be a graph of translation pairs that is shared by both models. If assumptions (a)-(b) hold, then the cross-lingual PELP and the PML are the identical in the limit as 
\[
\lambda_1 \to \infty\,.
\]
\end{restatable}
\begin{proof2}
See appendix.
\end{proof2}

\begin{restatable}{prop}{propositionfive} \label{prop:cross-lingual_b}
Let the cross-lingual PELP as defined in  Definition \ref{definition:cross-lingual_pelp} and the orthogonal mapping method \citet[OMM][]{artetxe2016learning} be defined as in \ref{eq:mapping_methods}. Also, let $\mathcal{G_{AB}}$ be a graph of translation pairs that is shared by both models. If assumptions (a)-(b) hold, then the MAP estimates of cross-lingual PELP and the OMM models are identical in the limit as 
\[
\lambda_1 \to 0\,.
\]
\end{restatable}
\begin{proof2}
See appendix.
\end{proof2}

\subsection{Flexibility of the PELP model}
 
Laplacian priors enable the PELP model to regularize relations between word and context embeddings, something used in Proposition \ref{prop:dict2vec}. Some side-information capture relatedness rather than the similarity between two words, which would favour word-context connections rather than word-word ones \citep{kiela2015specializing}. Moreover, PELP is defined both for unweighted and weighted graphs. This general structure can incorporate side-information of different confidence or strength.

In addition, as was shown in Section \ref{sec:relationship_with_previous_word_embeddings}, many previously proposed models are essentially special cases of the PELP model. This fact enables us to straightforwardly combine previous models into new joint models of interest that combine different group, dynamic, graph or cross-lingual structures. Indeed, we utilize this in Section \ref{sec:experiments}, where we use a weighted Laplacian to create a dynamic model for sociolects and introduce a probabilistic cross-lingual embedding model.

\section{Experiments}
\label{sec:experiments}

The purpose of our empirical experiments is twofold. First, we want to demonstrate how PELP can easily be applied to a novel social science use case. Secondly, we want to empirically establish the connections to other theoretically derived models to understand potential discrepancies better. The PELP model has been implemented using TensorFlow to simplify the analysis of large corpora using GPU parallelism. The implementation is general in that any weighted or non-weighted Laplacian matrix can be used for large corpora with large vocabularies and is speed-wise on par with previous implementations. The code we used to run the experiments is available at \url{https://anonymous.4open.science/r/pelp-1740}, while details on data and hardware are elaborated in the supplementary material.

\subsection{Dynamic Party Affiliation Embeddings}
Many corpora have a natural division into multiple subsets of interest. It is possible to configure PELP for such settings, for example utilizing the connection to grouped models. Using the US Congress speech corpus \citep{gentzkow2018congressional}, we compare the language the Republicans and Democrats use in their speeches in a grouped and grouped-dynamic setting using our PELP implementation to study political polarization, an important area in political science \citep{monroe_colaresi_quinn_2017}.

Our approach assigns each word one vector per partition of the data set (e.g. the vectors $\rho_{cat/R}$, $\rho_{cat/D}$ for the word \texttt{cat}) while having one set of context vectors for the whole dataset. We regularize the model with a Laplacian prior. As each string of characters has been assigned two different word vectors, we connect them to form a bipartite graph between the two parties. The set of edges $E$ would consist of these pairs: $(cat/D, cat/R), \dots, (dog/D, dog/R)$. The Laplacian prior is then constructed based on this graph to form a group-dynamic probabilistic embedding model.

\begin{table}[h!]
\centering
\begin{tabular}{ c c } 
\hline
\textbf{PELP} & \\
\hline
democrat, abortion, &\\
announce, maine, republican, &  \\
gun, illegal, republicans,  &  \\
breaks, stimulus, taxes, immigration, &  \\
kentucky, accounting, wyoming &  \\
\hline
\textbf{Bernoulli} & \\
\hline
hubbert, islanders, rickover,  &\\ 
pottawatomi, gaspee, mastercard, &  \\ 
morgenthau, compean, 205106150, &\\
fairtax, vertical, peaking, & \\ 
follette, isna, hubberts & \\
\hline
\end{tabular}
\caption{Top 15 words $w$ with the largest distance between $\rho_{w,R}$ and $\rho_{w,D}$. PELP model above, reference model Bernoulli Embeddings below. }
\label{table_largest_cosine}
\end{table}
%
After obtaining the point estimates, we explored the vector space using euclidian distance as in \citet{rudolph2017structured}. We then calculated the most differing words between the parties. As seen in Table \ref{table_largest_cosine}, the estimated polarizing words are plausible and include topics of contention such as 'taxes', 'gun' and 'abortion', as well as party names and states. The reference model without a Laplacian prior picked up rare words with some plausible group differences but limited saliency, limiting its use in applied analysis. 

\begin{table}[h!]
\centering
\begin{tabular}{ c c } 
\hline
\textbf{Years} & \textbf{Most different word vectors} \\
\hline
\hline
'04/05 & \begin{tabular}{c} 1983, regulatory, speaker, \\ announcer, administration, union, \\ routine,  administrations, reagan, land \end{tabular} \\
\hline
'06/07 & \begin{tabular}{c} declaration, heinz, legal, kuwait, \\ drug, pennsylvania, john, strike, \\  administration, pharmaceutical \end{tabular} \\
\hline
'08/09 & \begin{tabular}{c} utah, hundred, relief, republican, \\ luxury, public, dr, nevada, \\ trade, savings \end{tabular} \\
\hline
'10/11 & \begin{tabular}{c} republican, patriot, maine, health, \\ obama, presidents, haiti, \\ democrats, republicans, debt \end{tabular}\\
\hline
\end{tabular}

\caption{Top 10 most different words $w$ with the largest Euclidian distance between $\rho_{w,R}$ and $\rho_{w,D}$ over time.}
\label{table:dynamic_sociolect_a}
\end{table}

Furthermore, we generalized the sociolect model to a dynamic setting by simply adding temporal edges (i.e. $(dog_{2012/D}, dog_{2013/D}),$ $\dots,$ $(dog_{2021/D}, dog_{2022/D})$  $\dots,$ $(dog_{2012/R}, dog_{2013/R}),$ $\dots,$ $(cat_{2012/D}, cat_{2012/D}) \dots$) in the graph. To the best of our knowledge, the dynamic-group embedding model has not been studied in the literature before. We ran the experiments for the speeches between 2004 and 2011, split into four two-year periods. For reference, we estimated a Bernoulli Embeddings model, which shared the context vectors similarly to PELP but did not use a Laplacian prior. Similarly to the static case, we calculated the most differing words between the two parties, which can be seen in Table \ref{table:dynamic_sociolect_a}.
\begin{table}[h!]
\centering
\begin{tabular}{ c c } 
\hline
\textbf{Years} & \textbf{Most different word vectors} \\
\hline
\hline
'04/05 & \begin{tabular}{c} hearingsto, eisentrager, pud, \\ au4, fulcher, shue, promarxist,  \\ 2022242878, 2245161, iryna\end{tabular} \\
\hline
'06/07 & \begin{tabular}{c} 30something, islander, \\ housedemocrats, cluding,\\ surrounds, cor, 224,\\  grip, periodically, ds \end{tabular} \\
\hline
'08/09 & \begin{tabular}{c} brainard, weisberg, calabro, \\ rickover, lael, 2244756, bina, \\ crossreferences, fadel, sd366 \end{tabular} \\
\hline
'10/11 & \begin{tabular}{c} outbuild, 205106150, lynden, \\ soriano, mastercard, intercity, \\ platon, shouldhave, \\ epicenters, hallow \end{tabular}\\
\hline
\end{tabular}

\caption{Reference model: Top 10 most different words $w$ with the largest Euclidian distance between $\rho_{w,R}$ and $\rho_{w,D}$ over time.}
\label{table:dynamic_sociolect_b}
\end{table}
Much of the results make intuitive sense from a social science perspective, considering the timeframe and partisan differences. For instance, the subprime crisis can be seen in the polarization of 'relief' in 2008-2009. The American healthcare reform of 2010 (the Affordable Care Act, ACA) is reflected in the polarization of 'health' in 2010-2011. Moreover, 'Reagan' is a polarizing word in the early 2000s, while 'Obama' reaches is more polarizing in 2010-2011. Some words are plausibly polarizing but not time-specific, e.g. 'republican', 'democrat', 'union', and 'debt'. As can be seen in Table \ref{table:dynamic_sociolect_b}, the reference model with the vague prior picked up a lot of noise, showing the importance of regularization in dynamic grouped-embedding models.

\begin{figure}[h]
\includegraphics[scale=0.35]{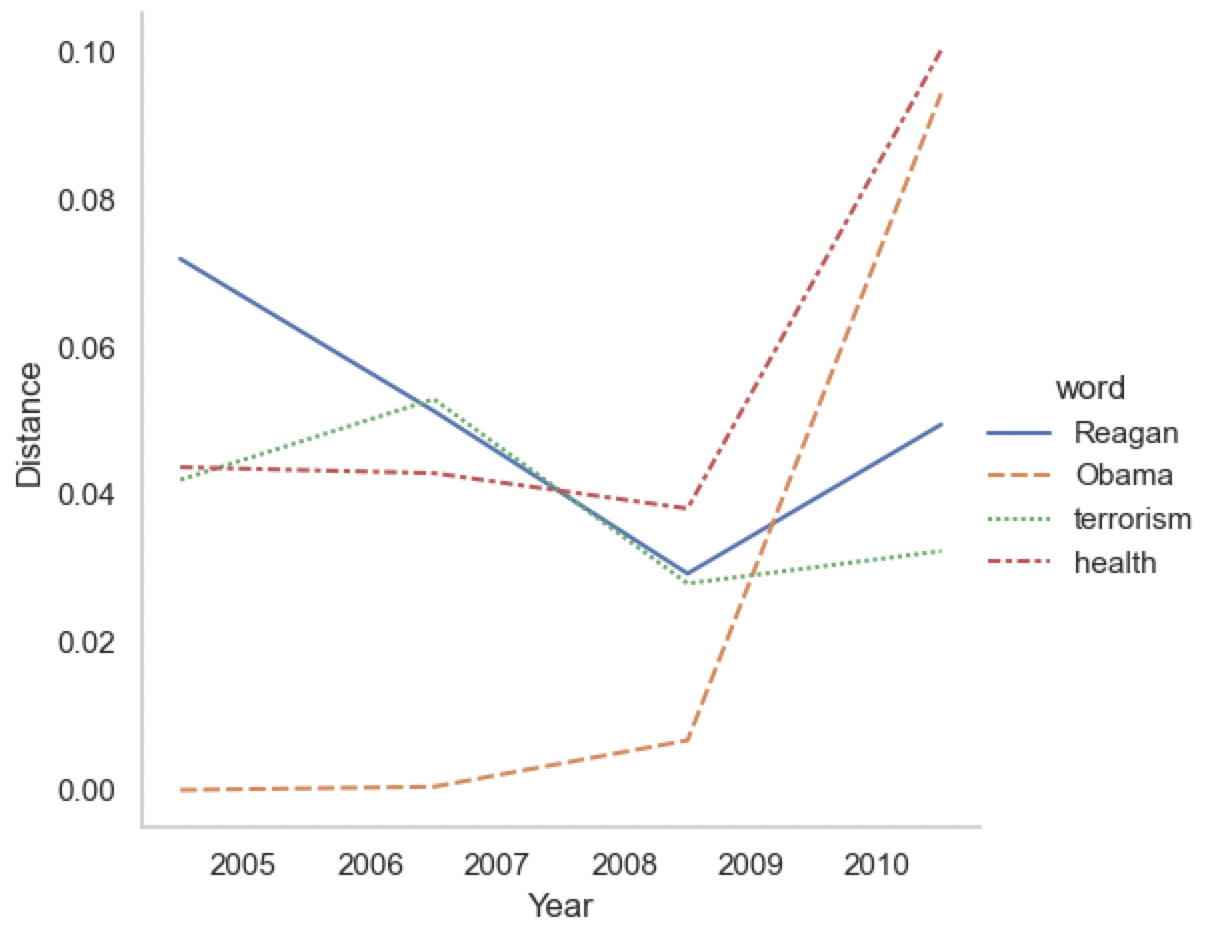}
\caption{The partisan difference between some words over time. Euclidian distance between $\rho_{v,R}$ and $\rho_{v,D}$.}
\label{presidents}
\end{figure}
Additionally, we visualized some polarizing words in Figure \ref{presidents}. We can see the temporal development of political polarization between the Republicans and the Democrats. The polarization of Reagan and Obama reflected the temporal distance from their terms in office. The polarization of 'health' rose during the period of the ACA, and 'Terrorism', on the other hand, slightly decreased during the same period.

\subsection{Improving Word Embeddings with Graph-Based Information}

Using \citet{tissier2017dict2vec}'s method to extract side-information, we experimented with augmenting the model with dictionary knowledge.

We downloaded all available dictionary definitions of the words in our vocabulary from Oxford, Collins, Cambridge and Dictionary.com dictionaries. The graph itself consisted of reciprocal mentions in the definitions, exemplified in Table \ref{table_word_definition}. As the words appeared in each other's \textit{definitions}, i.e. \textit{contexts}, we linked word and context vectors. Given an edge between the words $v$ and $w$, the word vector $\rho_v$ and the context vector $\alpha_w$ were connected in $\mathcal{G}$, and vice versa. This
%
is similar to the word-context connections in \citet{tissier2017dict2vec}.

As a baseline, we ran the \texttt{dict2vec} \citep{tissier2017dict2vec} model and a standard \texttt{word2vec} SGNS \citep{mikolov2013distributed} implementation, both provided by \citet{tissier2017dict2vec}.
We only used what \citet{tissier2017dict2vec} calls 'strong pairs', as specified before, of simplicity. The 'weak pairs' in their model were omitted in both \texttt{dict2vec} and in the PELP model. The \texttt{word2vec} implementation was run under the standard configuration. As we compared PELP's performance to \texttt{dict2vec}, we used a similar hyperparameter search over $\lambda_0 \in \{0.1, 0.5, 1, 2.5, 5\}, \lambda_1 \in \{1, 2, 4, 6, 8, 10, 12\}$ to find optimal performance on word similarity tasks as in \citet{tissier2017dict2vec}.

We used the same evaluation scheme as \citet{tissier2017dict2vec} used for \texttt{dict2vec}. This consisted of calculating the Spearman rank correlation between human evaluation and the distances in the embedding space. The same evaluation sets were used, namely MC-30 \citep{miller1991contextual}, MEN \citep{bruni2014multimodal}, MTurk-287 \citep{radinsky2011word}, MTurk-771 \citep{halawi2012large}, RG-65 \citep{rubenstein1965contextual}, RW \citep{luong2013better}, SimVerb-3500 \citep{gerz2016simverb}, WordSim-353 \citep{finkelstein2001placing} and YP-130 \citep{yang2006verb}.

\begin{table}[h!]
\centering
\begin{tabular}{ c c c c c } 
\hline
\textbf{Dataset} & \textbf{PELP} & \textbf{dict2vec} & \textbf{word2vec} \\
\hline
Card-660 & .284 & \textbf{.334} & .247 \\
MC-30 & \textbf{.778} & .761 & .673 \\
MEN-TR-3k & \textbf{.723} & .686 & .636 \\
MTurk-287 & \textbf{.670} & .575 &  .600 \\
MTurk-771 & \textbf{.664} & .606 & .564 \\
RG-65 & \textbf{.802} & .793 & .594 \\
RW-Stanford & \textbf{.447} & .441 & .387 \\
SimLex999 & .361 & \textbf{.393} & .321 \\
SimVerb-3500 & .270 & \textbf{.315} & .195 \\
WS-353-ALL & \textbf{.740} & .686 & .635 \\
WS-353-REL & \textbf{.697} & .626 & .559 \\
WS-353-SIM & \textbf{.756} & .718 & .685 \\
YP-130 & .476 & \textbf{.489} & .250 \\
\hline
\end{tabular}
\caption{Word similarity performance per evaluation dataset. Runs on the 50M token Wikipedia dataset.}
\label{table_word_similarity_2}
\end{table}

As expected, both PELP and \texttt{dict2vec} substantially outperformed the \texttt{word2vec} baseline on word similarity tasks.
As seen in Table \ref{table_word_similarity_2}, 
PELP had a slightly higher rank correlation with human evaluation than \texttt{dict2vec} on most evaluation sets. However, the overall difference is, as expected from our theory, small.

\begin{table}[h!]
\centering
\begin{tabular}{ c c c c c } 
\hline
\textbf{Dataset} & \textbf{PELP} & \textbf{dict2vec} & \textbf{word2vec}  \\
\hline
Card-660 & .362 & \textbf{.470} & .410 \\
MC-30 & \textbf{.828} & .700 & .785 \\
MEN-TR-3k & \textbf{.757} & .715 & .708 \\
MTurk-287 & \textbf{.696} & .605 & .634 \\
MTurk-771 & \textbf{.680} & .648 & .598 \\
RG-65 & \textbf{.823} & .778 & .740 \\
RW-Stanford & \textbf{.428} & .425 & .397 \\
SimLex999 & .382 & \textbf{.399} & .337 \\
SimVerb-3500 & .268 & \textbf{.323} & .216 \\
WS-353-ALL & \textbf{.737} & .680 & .702  \\
WS-353-REL & \textbf{.686} & .636 & .652 \\
WS-353-SIM & \textbf{.763} & .715 & .747 \\
YP-130 & \textbf{.519} & \textbf{.519} & .301 \\
\hline
\end{tabular}
\caption{Word similarity performance per evaluation dataset. Runs on the 200M token Wikipedia dataset.}
\label{table_word_similarity_2}
\end{table}

Interestingly, \citet{faruqui2014retrofitting} had found poor results applying the squared error regularization with a similarity graph during training. PELP's similar regularization, on the other hand, improved performance. One possible explanation for this is that both PELP in our experiments and \texttt{dict2vec} use \textit{word-context} connections instead of connections between word vectors. 

\subsection{Cross-Lingual PELP} \label{cross-lingual_laplacian_embeddings}
To assess the cross-lingual PELP presented in Section \ref{PELP_section}, we estimated bilingual embeddings for English-Italian. We counted the 5000 most common Italian words from a Wikipedia corpus, and their English translations were obtained from \texttt{translate.google.com}, similarly to \citet{mikolov2013exploiting}. The 5001th through 6000th most common words were used for evaluation.

As a baseline, we perform the orthogonal mapping method on vectors estimated with the Bernoulli Embeddings model of \citet{rudolph2016exponential}. We use the standard settings, both the context size and negative samples being five \citep{mikolov2013distributed}. Additionally, we conduct a grid search over the hyperparameters of PELP and select the best performing one as in previous experiments. 

To evaluate the accuracy of the cross-lingual embeddings, we used a Bilingual Lexicon Induction (BLI) task, i.e. inducing word translations from the shared vector space. The induction is deemed successful if the correct translation is included in the $k$ nearest neighbours by cosine distance, $k$ being the \textit{precision level} \citep{ruder2019survey}. We calculated the BLI scores at precision levels 1, 5 and 15 as percentage accuracy.

\begin{table}[h!]
\centering
\begin{tabular}{ c c c c } 
\hline
\textbf{Model} & \textbf{P@1} & \textbf{P@5} & \textbf{P@15} \\
\hline
$\text{PELP}_{\lambda_1=10^6}$ & 14.8 & 38.9 & 55.2 \\
$\text{PELP}_{\lambda_1=10^4}$ & 15.5 & \textbf{43.2} & \textbf{57.7} \\
Orthogonal & \textbf{16.0} & 32.2 & 40.3 \\
\hline\hline
$\text{PELP}_{\lambda_1=10^6}$ & 27.5 & 49.2 & 58.7 \\
$\text{PELP}_{\lambda_1=10^4}$ & \textbf{29.5} & \textbf{51.0} & \textbf{61.5} \\
Orthogonal & 22.8 & 37.8 & 44.6 \\
\hline
\end{tabular}

\caption{Bilingual Lexicon Induction (BLI)  performance for EN \textrightarrow IT (above) and IT \textrightarrow EN (below) at 5000 word pairs.}
\label{BLIscores5000}
\end{table}

As can be seen in Table \ref{BLIscores5000}, 
at 5000 translation pairs, PELP performed slightly better than the orthogonal mapping method. Translations from Italian to English were better than vice versa, probably because Italian has a larger vocabulary and thus more target classes in the task. As seen in Table \ref{BLIscores5000}, PELP was better on most precision level-direction combinations. The observed differences are understandable as PELP only approaches the orthogonal method in the limit $\lambda_1 \to 0$ (see Proposition \ref{prop:cross-lingual_b}). In a sense, PELP seems to be a good intermediate model between the orthogonal method and parameter sharing one, as PELP with a moderate value of $\lambda_1$ yielded better results than the orthogonal and PELP with high values of $\lambda_1$ (closer to parameter sharing). 

\begin{figure}[h]
\includegraphics[scale=0.40]{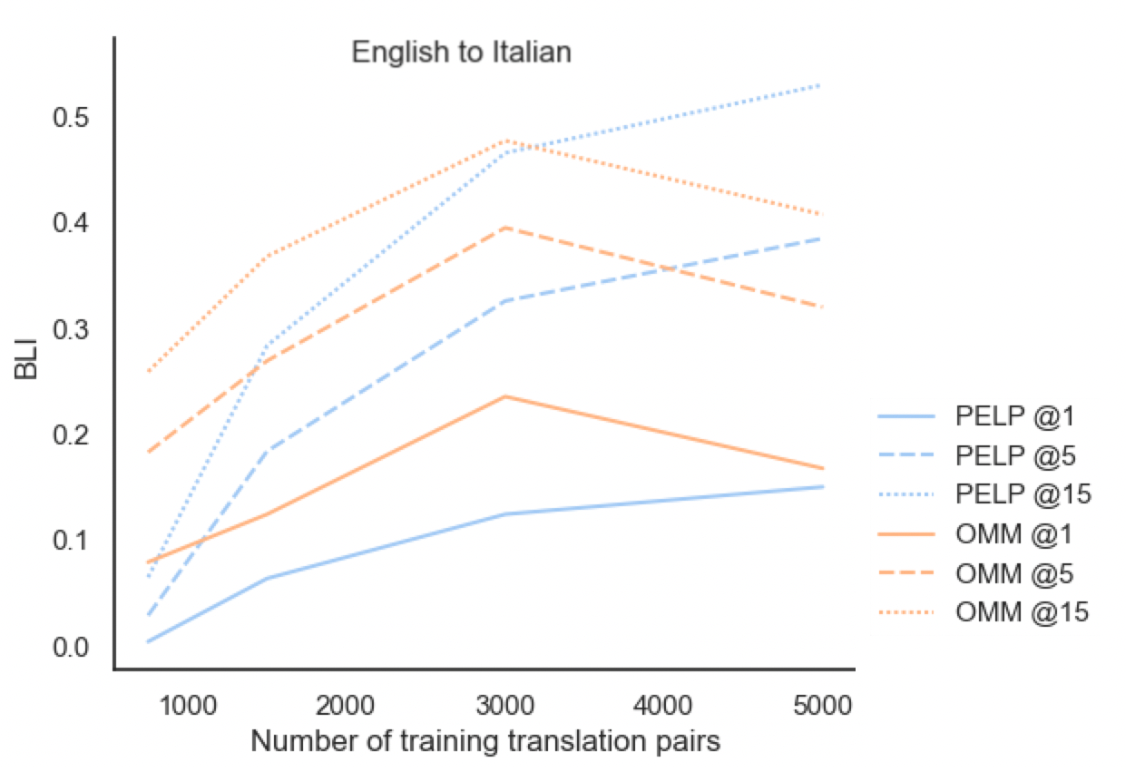}
\caption{Bilingual Lexicon Induction (BLI) performance by model, precision level and the number of translation pairs.}
\label{BLIscoresByPairs}
\end{figure}

As shown in Figure \ref{BLIscoresByPairs}, both models mostly benefitted from having more translation pairs. With a smaller number of translation pairs, PELP performed slightly worse than the orthogonal method. The lower performance at fewer translation pairs is possibly an artefact of the optimization process, as a small number of translation pairs is a relatively weak signal compared to the likelihood part of the posterior.

\section{Conclusion}

In this paper, we introduced a general probabilistic word embedding model with graph-based priors (PELP), which unifies many previous embedding methods under one general umbrella. We demonstrated the flexibility of the PELP model in a social science application by using the model in a novel use case, analyzing political sociolects over time. Moreover, we showed that this single model performs empirically on-par or better than many previous word embedding methods on several different tasks, both monolingual and crosslingual.

PELP's generality opens up methodological development for a large class of embedding models by working with the PELP model. Examples of such future work include improving the estimation and estimating parameter uncertainty.
Furthermore, building on the results of \citet{hofmann2020dynamic}, we believe it is possible to extend our results to contextualized embeddings in future work. This would provide important insight, as it is not clear whether contextualized embeddings are necessarily superior in inference settings \citep{rodriguez2021embedding}. The current model is formalized as a fully probabilistic model, which restricts the side-information to undirected graphs. It is possible to treat the regularization as a penalization problem instead, and use a non-symmetric matrix $\bfit{L}$, i.e. directed graphs. We leave this generalization for future work.
%

%

\bibliography{custom}

\appendix
\providecommand{\upGamma}{\Gamma}
\providecommand{\uppi}{\pi}
\section{Proofs}
Here we list assumptions that will be used throughout the proofs.
\begin{itemize}
    \item[(a)] The likelihood functions of two models that are being compared are identical, i.e. both are CBOW or SGNS as defined in Section \ref{background}.
    \item[(b)] There exists a \textit{maximum a posteriori} estimate $\hat{\theta} = {\arg \max}_\theta \log (\theta|x)$.
\end{itemize}

\subsection{Proof of Proposition \ref{prop:dict2vec}} \label{proof:dict2vec}

\propositionone*

\begin{proof}
The posterior of the regularized Dict2Vec model is
\begin{equation}
p(\theta \mid x) \propto p(x \mid \theta) p(\theta)
\end{equation}
which has the logarithmic form
\begin{equation}
\log p(\theta \mid x) = \log p(x \mid \theta) + \log p(\theta) + C
\end{equation}
where $C$ is a constant. Substituting the positive samples of the word pairs, we obtain
\begin{equation}
\begin{split}
\log  p(\theta \mid x) &= \log p(x \mid \theta) \\
&+ \sum_{(v,w) \in E} \log \sigma(\theta_{\alpha(v)}^T \theta_{\rho(w)}) \\
&+ \log p(\theta) + C
\end{split}
\end{equation}
where $\theta_{\alpha(v)} = \alpha_v$ and $\theta_{\rho(w)} = \rho_v$.

Differentiating the log posterior yields
\begin{equation} \label{dict2vecProof1}
\begin{split}
\frac{\partial}{\partial \theta_{\alpha(v)}} \left( \log p(\theta \mid x) \right) &= \frac{\partial}{\partial \theta_{\alpha(v)}} \left( \log p(x \mid \theta) \right) \\
&+ \sum_{(v,w) \in E} \sigma^\prime(\theta_{\alpha(v)}^T \theta_{\rho(w)}) \theta_{\rho(w)} \\
&- \gamma_0 \theta_{\alpha(v)} \\
&= \frac{\partial}{\partial \theta_{\alpha(v)}} \left( \log p(x \mid \theta) \right) \\
&+ \sum_{(v,w) \in E} l_{v,w} \theta_{\rho(w)} \\
&- \gamma_0 \theta_{\alpha(v)} \\
&= \frac{\partial}{\partial \theta_{\alpha(v)}} \left(\log p(x \mid \theta) \right) \\
&+ \sum_{(v,w) \in E} l_{v,w} \left( \theta_{\alpha(v)}-  \theta_{\rho(w)} \right) \\
&- \left( \gamma_0 - \sum_{(v,w) \in E} l_{v,w} \right) \theta_{\alpha(v)}
\end{split}
\end{equation}
where $\bfit{L}_{d2v}$ is the Laplacian matrix defined by the weights $l_{v,w}$ and $\bfit{D}$ is the diagonal matrix where $d_{v,v} = \gamma_0$ if there are no connections for the word $v$ and $d_{v,v} = \gamma_0 - \sum l_{v,w}$ if there are any.

On the other hand, a weighted SGNS PELP model has the following gradient
\begin{equation}
\begin{split}
\frac{\partial}{\partial \theta_{\alpha(v)}} \left( \log p(\theta \mid x) \right) &= \frac{\partial}{\partial \theta_{\alpha(v)}} \Big( \log p(x \mid \theta) \\
&- \frac{1}{2}\Tr \left( \theta^T (\bfit{L} + \bfit{D}) \theta \right) + C\Big) \\
&=\frac{\partial}{\partial \theta_{\alpha(v)}} \Big( \log p(x \mid \theta) \Big) \\
&+ \sum_{(v,w) \in E} l_{v,w} (\theta_{\alpha(v)} - \theta_{\rho(w)}) \\
&- d_{v,v} \theta_{\alpha(v)}
\end{split}
\end{equation}

Setting $\lambda_0 = \gamma_0 - \sum_{(v,w) \in E} l_{v,w}$ and $\bfit{L}^\star = \bfit{L}_{d2v} $ these gradients are the same. Note that you can swap $\rho(w)$ and $\alpha(v)$ and the relation holds.

Moreover, both loss functions are continuous in the whole domain, and they approach minus infinity in all directions as the magnitude of theta grows. Thus, the maximum is found at a critical point of the loss function of either model. Since the other loss function has identical gradients at any arbitrary point, it also shares the zero gradient at this maximum.
\end{proof}

\subsection{Proof of Proposition \ref{prop:dynamic}: Dynamic model} \label{proof:dynamic}
\propositiontwo*

\begin{proof}
The prior for an individual word type $v$ over the all timesteps $t \in \{1, 2, \dots,  T\}$ is

\begin{equation}
\begin{split}
\log {p_{v}(\theta)} &= -\frac{\lambda_0}{2} \norm{\theta_{v,0}}^2 -\frac{\lambda_1}{2} \sum_{t=0}^{T-1} \norm{\theta_{v,t} - \theta_{v,t+1}}^2 + C\\
&= -\frac{\lambda_0}{2} \norm{\theta_{v,0}}^2 + -\frac{\lambda_1}{2} \Tr{\left( \theta_v^T \bfit{L}_v \theta_v \right)} + C
\end{split}
\end{equation}
where $\bfit{L}_v$ is the Laplacian matrix of a graph with the edges $E_v = \{(0,1), (1,2), \dots, (T-1,T) \}$ associated with the word $v$, and $C$ is a constant. Moreover, the first term can be represented as a trace of a diagonal matrix
\begin{equation}
\begin{split}
\log {p_{v}(\theta)} &= -\frac{\lambda_0}{2} \norm{\theta_0}^2+ -\frac{\lambda_1}{2} \Tr{\theta^T \bfit{L}_v \theta} + C\\
&= -\frac{1}{2} \Tr{ \theta^T (\lambda_0 \bfit{D}_v + \lambda_1 \bfit{L}_v )\theta} + C
\end{split}
\end{equation}
where
\begin{equation}
\bfit{D}_v = \begin{bmatrix}
1 & 0 & \dots & 0\\
0 & 0 & & \vdots \\
\vdots &  & \ddots & 0 \\
0 & \dots & 0 & 0
\end{bmatrix} 
\in \mathbb{R}^{D \times D}
\end{equation}
Based on the prior on the individual word types, construct a Laplacian for all word types

\begin{equation}
\bfit{L}^\star = \begin{bmatrix}
\bfit{L}_1 & 0 & \dots & 0\\
0 & \bfit{L}_2 & & \vdots \\
\vdots &  & \ddots & 0 \\
0 & \dots & 0 & \bfit{L}_V
\end{bmatrix} 
\end{equation}

\begin{equation}
\bfit{D}^\star = \begin{bmatrix}
\bfit{D}_1 & 0 & \dots & 0\\
0 & \bfit{D}_2 & & \vdots \\
\vdots &  & \ddots & 0 \\
0 & \dots & 0 & \bfit{D}_V
\end{bmatrix} 
\end{equation}
and present the a priori distribution of theta as a Gaussian with Laplacian matrix plus a diagonal matrix
\begin{equation}
\log {p(\theta)} = -\frac{1}{2} \Tr{\Big( \theta^T(\bfit{L}^\star + \bfit{D}^\star )\theta \Big)} + C
\end{equation}
as the precision matrix.

\end{proof} 

\subsection{Proof of Proposition \ref{prop:grouped}: Grouped Model}
\label{proof:grouped}

\propositionthree*

\begin{proof}
A word type $v$ in a GBM has the following log prior

\begin{equation} \label{groupedPrior}
\begin{split}
\log p_v(\theta) &= -\frac{\lambda_0}{2} \norm{\theta_v^\prime}^2 -\frac{\lambda_1}{2} \sum_{s \in S} \norm{\theta_{s,v}- \theta_v^\prime}^2 + C
\end{split}
\end{equation}

As the group parameters do not appear in the likelihood function, their MAP estimates can be obtained analytically

\begin{equation}
\hat{\theta}_v^\prime = \gamma \sum_{s \in S} \theta_{s,v}
\end{equation}
where
\begin{equation}
    \gamma = \frac{\lambda_1}{\lvert S \rvert \lambda_1 + \lambda_0}
\end{equation}

Substituting the analytically obtained $\hat{\theta}_v^\prime$ to Equation \ref{groupedPrior}, its MAP estimate can be formulated as a sum of squared differences of the parameters plus norms of the parameters
\begin{equation}
\begin{split}
\log p_v(\theta) &= -\frac{\lambda_0}{2} \norm{\theta_v^\prime}^2+ -\frac{\lambda_1}{2} \sum_{s \in S} \norm{\theta_{s,v}- \theta_v^\prime}^2 \\
&= -\frac{\lambda_0}{2} \norm{\gamma \sum_{s \in S} \theta_{s,v}}^2 -\frac{\lambda_1}{2} \sum_{s \in S} \norm{\theta_{s,v}- \gamma \sum_{s \in S} \theta_{s,v}}^2 \\
&= a \sum_{s_1,s_2 \in S, s_1 \neq s_2} - \theta_{s_1,v}^T \theta_{s_2,v} + b \sum_{s \in S} \norm{\theta_{s,v}}^2
\end{split}
\end{equation}
where $a,b$ are negative constants. This in turn can be organized into

\begin{equation}
\log p_v(\theta) = c \sum_{s \in S} \norm {\theta_{s,v}}^2 + d \sum_{s_1,s_2 \in S} \norm{\theta_{s_1,v} - \theta_{s_2,v}}^2
\end{equation}
where $c,d$ are negative constants. This corresponds to the Laplacian prior

\begin{equation}
\log p_v(\theta) = -\frac{1}{2} \Tr( \theta^T \left(a \bfit{I} + b\bfit{L} \right) \theta) + C
\end{equation}
where $\bfit{I}$ is a $\mathbb{R}^{D \times D}$ identity matrix and $\bfit{L}$ is the Laplacian of a graph where the groups are fully connected. Based on the priors on the individual word types $v \in W$, construct a Laplacian for all groups

\begin{equation}
\bfit{L}^\star = \begin{bmatrix}
\bfit{L}_1 & 0 & \dots & 0\\
0 & \bfit{L}_2 & & \vdots \\
\vdots &  & \ddots & 0 \\
0 & \dots & 0 & \bfit{L}_V
\end{bmatrix} 
\end{equation}

\begin{equation}
\bfit{D}^\star = \begin{bmatrix}
a_1 &  \dots & 0\\
\vdots &  \ddots & \vdots \\
0 & \dots & a_V
\end{bmatrix} \otimes \bfit{I}
\end{equation}

and present the a priori distribution of theta

\begin{equation}
\log {p(\theta)} = -\frac{1}{2} \Tr{\Big( \theta^T(\bfit{L}^\star + \bfit{D}^\star )\theta \Big)} + C
\end{equation}
as a Gaussian with Laplacian matrix plus a diagonal matrix as the precision matrix.

\end{proof}

\subsection{Proof of Proposition \ref{prop:cross-lingual_a}: cross-lingual model A} \label{proof:cross-lingual_a}

\propositionfour*

\begin{definition}
cross-lingual PELP over the two languages $A,B$, the parameters are $\rho_A \in \mathbb{R}^{V_A \times D}$, $\rho_B \in \mathbb{R}^{V_B \times D}$, $\alpha_A \in \mathbb{R}^{V_A \times D}$, $\alpha_B \in \mathbb{R}^{V_B \times D}$. They constitute $\theta = [\rho_A, \rho_B, \alpha_A, \alpha_B] \in \mathbb{R}^{2(V_A + V_B) \times D}$

The likelihood is then defined as
$$
p(x \mid \theta) = p(x_A \mid \theta_A) + p(x_B \mid \theta_B)
$$

While the prior $\Tr(\theta^T \bfit{L}_+ \theta)$ is defined by the augmented Laplacian with the graph of translation pairs, applied on both word ($\rho_{Ai} \sim \rho_{Bj}$) and context vectors.
\end{definition}

\begin{proof}
The posterior of a cross-lingual PELP model can be factorized into the likelihood, and the prior on $\rho$ and $\alpha$
\begin{equation}
p(\rho, \alpha \mid x) \propto p(x \mid \rho, \alpha)p(\rho)p(\alpha)
\end{equation}
where x is data. The Laplacian prior on $\rho$ is
\begin{equation}
p(\rho) \propto \exp(-\Tr{\rho^T( \lambda_1 \bfit{L} + \lambda_0 \bfit{I})\rho})
\end{equation}
and can be further factorized into
\begin{equation}
p(\rho) \propto \exp(-\lambda_0 \Tr{\rho^T \bfit{I}\rho}) \exp(-\lambda_1 \Tr{\rho^T \bfit{ L}\rho})
\end{equation}

As $\lambda_1 \to \infty $, the latter factor of the prior approaches zero
\begin{equation}
\lim_{\lambda_1 \to \infty} \exp{\Big(-\lambda_1 \sum_{(v,w) \in E} \norm{\rho_v - \rho_w}^2 \Big)} = 0
\end{equation}
iff
\begin{equation}
\sum_{(v,w) \in E} \norm{\rho_v - \rho_w}^2 > 0
\end{equation}
where $E$ is the set of translation pairs. This sum is zero if and only if all word vectors of the translation pairs are equal to each other. Moreover, if the sum is zero, the prior is 1. Thus, the posterior is nonzero only when this is the case, forcing the vectors to be the same. As showed in \citet{ruder2019survey}, the model presented in \citet{gouws2015simple} forces the translation pairs to be the same in MAP estimation. Moreover, \citet{gouws2015simple}'s model uses the same CBOW loss function as PELP. Thus, PELP and \citet{gouws2015simple}'s model are equivalent in the limit.
\end{proof}

\subsection{Proof of Proposition \ref{prop:cross-lingual_b}: cross-lingual model B}
\label{proof:cross-lingual_b}

\propositionfive*

\begin{proof}
The posterior of a PELP model can be factorized
\begin{equation}
\begin{split}
p(\theta \mid x) &\propto p(x \mid \theta)p(\theta) \\
&\propto p(x \mid \theta)\exp(\Tr{\theta^T(\bfit{L} + \bfit{D})\theta}) \\
&= p(x \mid \theta)\exp(\Tr{\theta^T\bfit{D}\theta}) p(\Tr{\theta^T\bfit{L}\theta}) \\
\end{split}
\end{equation}
where x is data. Let the posterior probability without the Laplacian factor be $p_0(\theta \mid x)$

\begin{equation}
\begin{split}
p_0(\theta \mid x) \propto p(x \mid \theta)\exp(\Tr{\theta^T\bfit{D}\theta}) 
\end{split}
\end{equation}
Let's further define its global optima as the set $H$

\begin{equation}
H = \{ \theta^\prime \in \mathbb{R}^{V \times D} \mid p_0(\theta^\prime \mid x) \geq p_0(\theta \mid x) \forall \theta \in \mathbb{R}^{V \times D} \}
\end{equation}
Moreover, let's denote the value of the global optimum $p_0(\theta^\prime \mid x)$ with a shorthand
\begin{equation}
- \log p_0(\theta^\prime \mid x) = M
\end{equation}
Now let PELP's log posterior $p_1(\theta \mid x)$ be
\begin{equation}
-\log p_1(\theta \mid x) = -\log p_0(\theta \mid x) + \lambda_1 \Omega (\theta)
\end{equation}
where the Laplacian prior $\log(\theta)$ is
\begin{equation}
\log(\theta) = \lambda_1 \Omega(\theta) + C
\end{equation}
and
\begin{equation}
\Omega(\theta) = \sum_{(v,w) \in E_{AB}} \norm{ \theta_v - \theta_w}^2
\end{equation}
is the log Laplacian prior. As $p_0(\theta \mid x)$ is invariant wrt. $\theta$ in the set $H$, the full posterior $p_1(\theta^\prime \mid x)$, $\theta^\prime \in H$ can be expressed in the form
\begin{equation}
-\log p_1(\theta^\prime \mid x) = M + \lambda_1 \Omega (\theta^\prime)
\end{equation}

Since $-\log {p_0(\theta)} > M$ (strictly) for all $\theta \notin H$, and $\Omega (\theta^\prime)$ has a maximum in the set $H$, for all $\theta \notin H$ there exists a $\delta > 0$ for which for all $\lambda_1 < \delta$ the inequality
\begin{equation}
-\log {p_0(\theta)} > M + \lambda_1 \Omega (\theta^\prime)
\end{equation}
holds. For all $\lambda_1 < \delta$ subsequently holds that
\begin{equation}
\begin{split}
-\log {p_1(\theta)} > -\log {p_0(\theta)} &> M + \lambda_1 (\theta^\prime) \\
\implies -\log {p_1(\theta)} &>  M + \lambda_1 \Omega (\theta^\prime) \\
\implies -\log {p_1(\theta)} &> -\log {p_1(\theta^\prime)}
\end{split}
\end{equation}
Thus, as $\lambda_1 \to 0$, no $\theta \notin H$ can be a global optimum of $p_1(\theta \mid x)$, so its global optimum has to be in the set $H$.

\begin{lemma}
$p_0$ of one language A is invariant wrt. orthogonal transformations $\bfit{W} \in D \times D,$ $\bfit{W}^T\bfit{W} = \bfit{I}$.
\end{lemma}
\begin{proof} The $\bfit{W}$ transformed $p_0$ on the language $A$ is
\begin{equation}
\begin{split}
p_0(\bfit{W}\theta \mid x_A) =& \left(\sum_{i=0}^N \sum_{(w,v) \in C(i)} \log \sigma((\bfit{W} \alpha_v)^T \bfit{W} \rho_w) \right) \\
&+ \sum_{v \in W} \norm{\bfit{W}\alpha_v}^2 + \sum_{v \in W} \norm{\bfit{W}\rho_v}^2 \\
 = &\left(\sum_{i=0}^N \sum_{(w,v) \in C(i)} \log \sigma( \alpha_v^T \bfit{W}^T \bfit{W} \rho_w) \right) \\
 &+ \sum_{v \in W} \alpha_v^T \bfit{W}^T \bfit{W} \alpha_v + \sum_{v \in W} \rho_v^T \bfit{W}^T \bfit{W} \rho_v \\
 = &\left( \sum_{i=0}^N \sum_{(w,v) \in C(i)} \log \sigma( \alpha_v^T \rho_w) \right) \\
 &+ \sum_{v \in W}\norm{\alpha_v}^2 + \sum_{v \in W} \norm{\rho_v}^2 \\
=& p_0(\theta \mid x_A)
\end{split}
\end{equation}
which is the original posterior.
\end{proof}

\begin{lemma}
The posterior $p_0$ is invariant wrt. orthogonal transformations that only applies to language A of the set of languages $\{A, B \}$.
\end{lemma}

\begin{proof}
Due to Lemma 1, 
\begin{equation}
    p_{0,A}(\bfit{W} \theta_A) = p_{0,A}(\theta_A)
\end{equation}
We can thus write the full with the transformation on language $A$
\begin{equation}
\begin{split}
\log p_{0}([\bfit{W} \theta_A, \theta_B] \mid x) = & \log p_{0,A}(\bfit{W} \theta_A \mid x_A)\\
&+ \log p_{0,B}(\theta_B \mid x_B) \\
= & \log p_{0,A}(\theta_A \mid x_A) \\
&+ \log p_{0,B}(\theta_B \mid x_B) \\
=& \log p_{0}([\theta_A, \theta_B] \mid x)
\end{split}
\end{equation}
\end{proof}

Thus, each individual global optimum of $p_0(\theta \mid x)$ corresponds to a set of orthogonal transformations of it. These rotations only affect the sum
\begin{equation}
\Omega(\theta, \bfit{W}) = \sum \norm{\bfit{W} \theta_v - \theta_w}^2
\end{equation}

In $H$ the $\Omega$-minimizing rotation $\bfit{W}$ is the optimum. Thus the optimum with the minimizing rotational matrix $\bfit{W}$ also minimizes the posterior in the set $H$.

The orthogonal mapping method first finds the monolingual optima $\theta^\prime$ of $p_{0,A}(\theta_A \mid x_A)$ and $p_{0,B}(\theta_B \mid x_B)$. Since the bilingual $p_0(\theta) =p_{0,A}(\theta_A \mid x_A) + p_{0,B}(\theta_B \mid x_B)$, this is equivalent to optimizing $p_0(\theta \mid x)$, and the concatenation of the two optima $[\theta^\prime_A, \theta^\prime_B] \in H$. Thus, this step is the same in both models.

It then selects the orthogonal mapping that minimizes $\Omega(\theta^\prime, \bfit{W})$ wrt. $\bfit{W}$. As shown, PELP selects the $\theta^\prime \in H$ whose orthogonal rotation minimizes $\Omega(\theta^\prime, \bfit{W})$ as well, making the two methods equivalent in the limit.

\end{proof}

\section{Data, hardware and runtimes}
For all experiments, a 580 RX GPU with 8GB of VRAM was used to run the experiments. TensorFlow ROCm 3.8 and TensorFlow Probability 0.11 were used. The runtimes ranged from roughly 2 to 8 minutes per epoch. This translates into 30 minutes to 2 hours of full runtime at 15 epochs. The implementation was tested to work on a CPU, though it is much faster on a GPU.

All data was converted to lowercase. Punctuation and other special characters were removed, though numbers were retained. For Italian, accentuated characters (e.g. 'è') were retained.

In the first experiment, US congress speech data was used. According to the provided metadata, it was split into Republican and Democrat speeches. For the static use case, a 120 million token subset was used. For the dynamic use case, a 116 million token subset was used, corresponding to 8 years of data.

In the second experiment, 50 and 200 million token partitions of the latest English Wikipedia were used.

In the third experiment, 50 million token partitions of the latest English and Italian Wikipedia were used. For these models, we ran 200 epochs to ensure convergence. This took approximately 9 hours on our hardware.

\end{document}